\documentclass[letterpaper, 10 pt, conference]{ieeeconf}  

\IEEEoverridecommandlockouts                              

\let\proof\relax                
\let\endproof\relax

\makeatletter
\let\NAT@parse\undefined
\makeatother

\usepackage{times}
\usepackage{graphicx}
\usepackage{etoolbox}
\usepackage{bm}
\usepackage{xspace} 
\usepackage{standalone} 
\usepackage{pgfplots} 
\usepackage{tikz}[includestandalone] 
\usepackage[T1]{fontenc} 
\usepackage[font=small]{caption}
\usepackage{mathtools}
\usepackage{enumerate}
\usepackage[ruled,vlined,noend]{algorithm2e}
\usepackage{amssymb,amsmath,amsthm}
\usepackage{todonotes}
\usepackage{comment}
\usepackage{ifthen}
\usepackage{siunitx}
\usetikzlibrary{math}
\usepackage{transparent}
\usepackage{sidecap}

\newtheorem{lemma}{Lemma}
\usepackage{algorithmic} 
\usepackage{textcomp} 
\usepackage{subcaption} 
\usepackage[numbers,sort&compress]{natbib}
\usepackage{multicol}
\usepackage{tabularx}
\usepackage[bookmarks=true]{hyperref}
\usepackage{xcolor}
\usepackage{siunitx}
\hypersetup{
  colorlinks,
  linkcolor={red!50!black},
  citecolor={blue!50!black},
  urlcolor={blue!80!black}
  }

\usepackage{lpform}

\let\emptyset\varnothing

\theoremstyle{definition}

\theoremstyle{plain}
\newtheorem{thm}{Theorem} 
\theoremstyle{definition}

\newtheorem{definition}[thm]{Definition}

\newcounter{tecounter}
\providecommand{\gobble}[1]{}

\providecommand{\Language}[1]{\ensuremath{\mathcal{L}(#1)}}

\newcommand*{\probleminternal}[4]{
	\par
	\medskip
	\noindent\fbox{\parbox{0.98\columnwidth}{
		\textbf{#4:} {#1} \\[0.05in]
		\renewcommand{\tabcolsep}{2pt}
		\begin{tabularx}{\linewidth}{rX}
			\emph{Input:} & #2 \\
			\emph{Output:} & #3
		\end{tabularx}
	}}
	\par
	\medskip
	\par
}

\let\emptyset\varnothing

\newcommand*{\ourproblem}[3]{\probleminternal{#1}{#2}{#3}{Problem}}

\providecommand{\defemp}[1]{\emph{#1}} 
\newcommand{\fm}{{\sc fm}\xspace}

\providecommand{\reachedvf}[3]{\ensuremath{\mathcal{V}_{#2}(#1, #3)}}
\providecommand{\reachedf}[2]{\ensuremath{\mathcal{V}_{#1}(#2)}}
\providecommand{\reachedv}[2]{\ensuremath{\mathcal{V}(#1, #2)}}
\providecommand{\reachedc}[2]{\ensuremath{\mathcal{C}(#1, #2)}}

\providecommand{\reaching}[2]{\ensuremath{\mathcal{S}^{#1}_{#2}}}

\providecommand{\negation}[1]{\ensuremath{\overline{#1}}}
\providecommand{\bland}{\ensuremath{\bigwedge}}
\providecommand{\blor}{\ensuremath{\bigvee}}
\providecommand{\IP}{{ILP}\xspace}
\providecommand{\SAT}{{SAT}\xspace}
\providecommand{\INP}{{INP}\xspace}
\providecommand{\LazySAT}{{LazySAT}\xspace}

\newenvironment{sproof}{%
  \proof}{\endproof}

\providecommand{\vdagger}{v^{\dagger}}

\providecommand{\zipped}{zipped\xspace}
\providecommand{\Zipped}{Zipped\xspace}
\providecommand{\Zip}{Zip\xspace}

\def\BibTeX{{\rm B\kern-.05em{\sc i\kern-.025em b}\kern-.08em
    T\kern-.1667em\lower.7ex\hbox{E}\kern-.125emX}}

\providecolor{noted}{rgb}{1,0,1}

\providecolor{renoted}{rgb}{0,1,0}

\providecolor{justsaying}{rgb}{1,0,0}

\usepackage[framemethod=default,innerleftmargin=5pt,innerrightmargin=5pt,skipabove=5pt]{mdframed}

\begin{document}

\title{\LARGE \bf 
Accelerating combinatorial filter reduction through constraints
}

\newcommand\unsure[1]{\textcolor{brown}{{#1}}}
\newcommand\oldtext[1]{}

\author{
Yulin Zhang,$^1$ Hazhar Rahmani,$^2$ Dylan A. Shell,$^1$ Jason M. O'Kane$^2$\vspace*{-5pt}
\thanks{$^1$Yulin Zhang and Dylan A. Shell are with the Dept. of Computer Science and Engineering, Texas A\&M University, College Station, TX, USA.  {\tt\small yulinzhang@tamu.edu, dshell@tamu.edu}}
\thanks{$^2$Hazhar Rahmani and Jason M. O'Kane are with the Dept. of Computer Science and Engineering, University of South Carolina, Columbia, SC, USA. {\tt\small hrahmani@email.sc.edu,jokane@cse.sc.edu}}
}

\maketitle
\begin{abstract}
Reduction of combinatorial filters involves compressing state representations
that robots use. Such optimization arises in automating the construction of
minimalist robots. But exact combinatorial filter reduction is an NP-complete
problem and all current techniques are either inexact or formalized with
exponentially many constraints. 
This paper proposes a new formalization needing only a polynomial number of
constraints, and characterizes these constraints in three different forms:
nonlinear, linear, and conjunctive normal form. Empirical results show that 
constraints in conjunctive normal form capture the
problem most effectively, leading to a method that outperforms the others. 
Further examination indicates that a substantial proportion of constraints
remain inactive during iterative filter reduction. To leverage this
observation, we introduce just-in-time
generation of such constraints, which yields improvements in 
efficiency and has the potential to minimize large filters.
%
\end{abstract}


\section{Introduction}
\label{sec:intr}
A growing body of work has described tools, employed optimization methods, or
proposed new algorithms to help automate the design and/or fabrication of robots
(e.g.,~\citep{Luck-RSS-17,Schulz17,HoovFearing08,pervan2018low,zhang2020abstractions}).  Important
among those approaches are algorithms that aim to manage or minimize resources
(for instance, see~\citep{censi17co}).  Memory is one resource of
particular interest to us, not because RAM is expensive, but rather because 
when state requirements are reduced this often conveys insight into the
fundamental informational structure of particular robot
tasks~(cf.~\citep{connell90,donald95information}). To this end, we focus on
filter reduction, targeting combinatorial filters of the style promoted by
LaValle~\citep{lavalle10sensing}.  These filters are discrete variants of the
probabilistic estimators ubiquitous in modern robotics, and they yield
particularly elegant treatments for certain practical tasks (e.g, see
\citep{tovar2014combinatorial}).  The minimization problem for combinatorial filters is
simple to state and easy to grasp, but continued work on the
topic~\citep{okane2013reduction,o2017concise,saberifar2017combinatorial,rahmani2018relationship,zhang20cover,rahmani2020integer}
shows that it involves more than first meets the eye. 

As a simple example to motivate filter minimization,
consider the safari park with vehicle rental service shown in
Figure~\ref{fig:motivating_scenario}. 
The cars for hire are each equipped with a compass and an intelligent gear shifting system. The compass measures the heading of the vehicle before and after its movement, e.g., `nw' means that the vehicle was heading north and then turned to face west. The intelligent gear shifting system takes the readings from the compass as input, and automatically shifts gears to abide with the speed limit.
There are three types of speed limits on roads: ($a$) between $15$ and $30$ (gray), ($b$) speeds below $30$ (brown), or ($c$) slower than $15$ (green). 
Every vehicle is capable of moving with a low gear to drive with a maximum speed $15$, and with a high gear to drive between speed $15$ and $30$.
A na\"ive gear shifting system satisfying the speed limits is realized by a filter shown in Figure~\ref{fig:motivating_input}: each vertex represents a system state, each edge represents the state transition, with the label on the edges representing the readings from the compass. 
A vertex is colored gray if the system outputs high gear, colored green if it outputs low gear, or colored both colors if the system may use either gear (via, say, a nondeterministic choice).
We are interested in finding a minimal filter, like the one shown in Figure~\ref{fig:motivating_minimal}, that realizes appropriate behavior but with fewest states.
\begin{figure}[t] 
\begin{subfigure}[b]{\linewidth}
\centering
\includegraphics[scale=0.3]{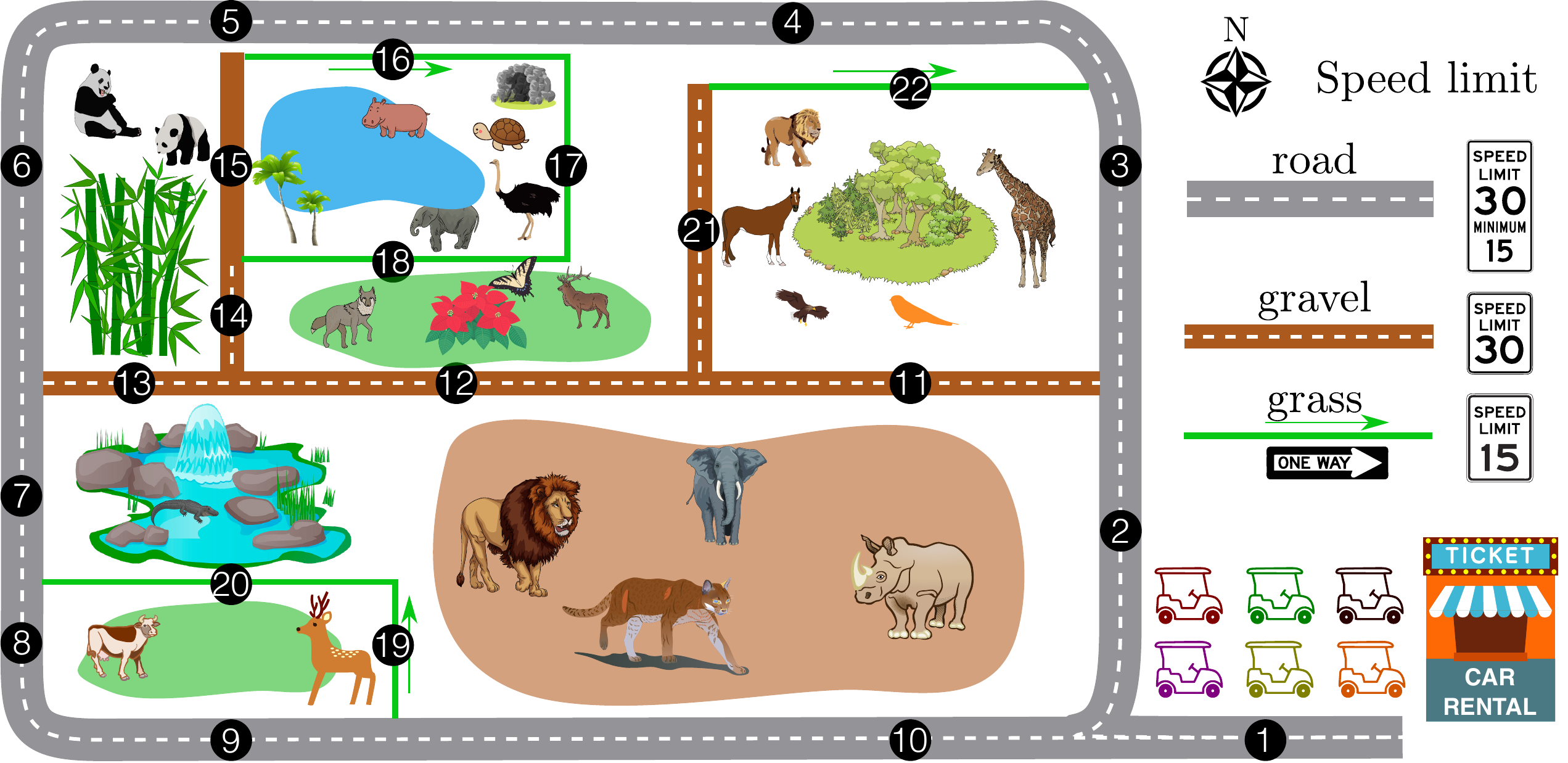}
\caption{\label{fig:motivating_scenario}}
\end{subfigure}
\begin{subfigure}[b]{0.58\linewidth}
\includegraphics[scale=0.23]{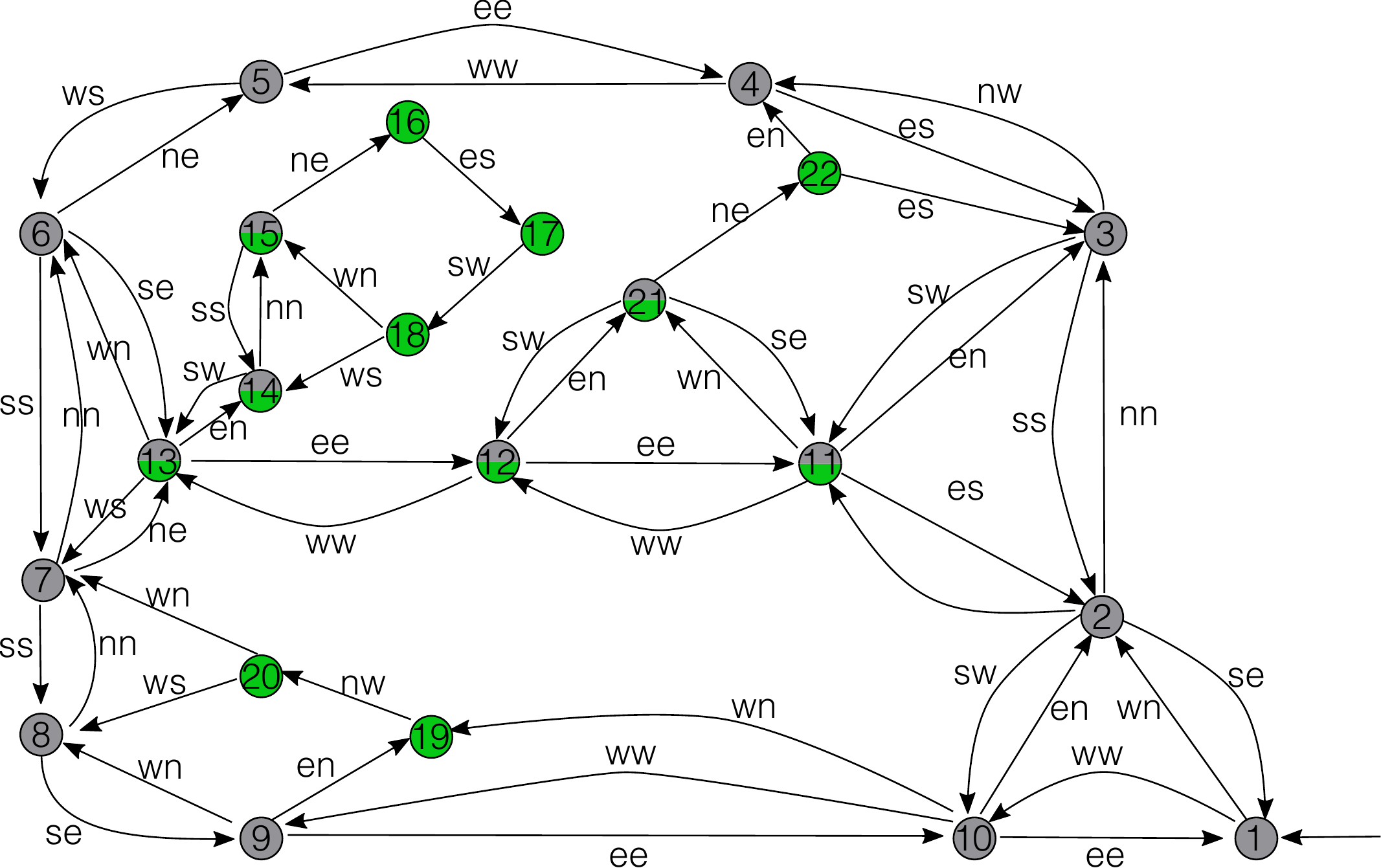}
\caption{\label{fig:motivating_input}} 
\end{subfigure}
\begin{subfigure}[b]{0.37\linewidth}
\includegraphics[scale=0.25]{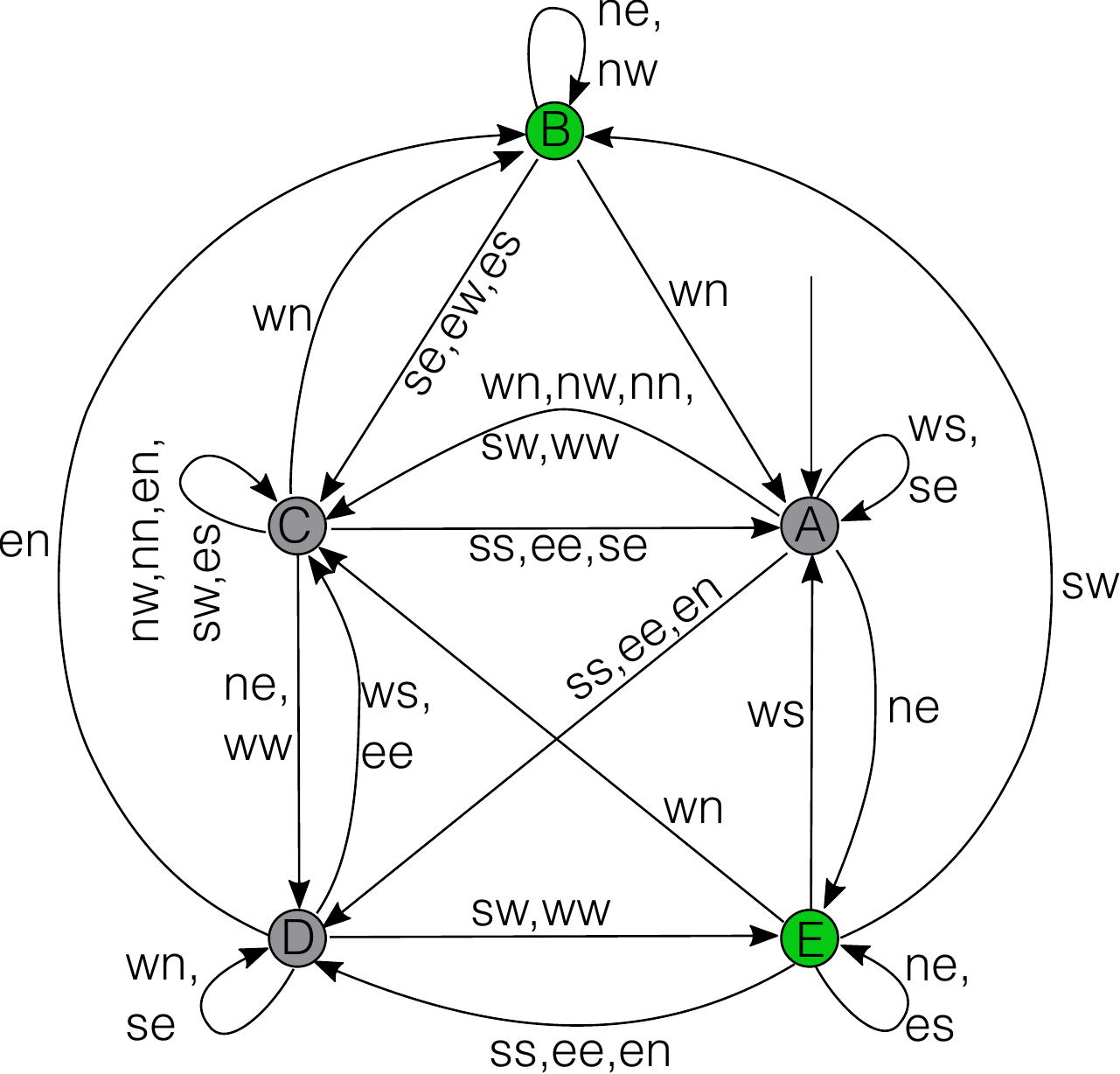}
\caption{\label{fig:motivating_minimal}} 
\end{subfigure}
\caption{(a) A safari park with vehicles for rental. The vehicles are
equipped with an intelligent gear shifting system automatically shifts the gear to satisfy the speed limit according to its compass readings. (b) A na\"ive filter to implement the intelligent gear shifting system. (c) A minimal filter for the intelligent gear shifting system.
\label{fig:motivating_excample}}
\vspace{-1cm}
\end{figure}

A natural way to proceed is by first constructing a discrete state-transition
system (such as in Figure~\ref{fig:motivating_input}) using the problem
description as a basis. Then, the next step is to apply some
algorithm capable of compressing it. 
Despite the apparent similarity of combinatorial
filter minimization to the problem of state minimization of deterministic
automata, with Myhill--Nerode's famous and efficient
reduction~\citep{hopcroft06introduction}, minimization of combinatorial filters
is NP-complete~\citep{okane2013reduction}. 
%
One thread of work studies filter minimization based on merger operations. 
These algorithms reduce filter minimization to a graph coloring instance~\citep{okane2013reduction,o2017concise,saberifar2017combinatorial} or integer linear programming~\citep{rahmani2020integer}. 
Saberifar et. al.~\citep{saberifar2017combinatorial} examined special cases, approximation and parameterized complexity of filter minimization. 
Rahmani and O'Kane~\citep{rahmani2018relationship} showed that the well-known notion of \emph{bisimulation} relation in general yields only sub-optimal solutions.
They proposed three different integer linear programming formulations to search for the smallest equivalence relation~\citep{rahmani2020integer}. 
Recent work shows that both merge and split operations are necessary to find a
minimal filter~\citep{zhang20cover}. It formalizes the problem as a clique
cover problem subject to the constraint that the output filter should be
deterministic and must simulate the output of the input filter. But both
criteria involve constraints that are exponential in size, which is obviously
computationally unattractive.

In this paper, we propose a new, competing formulation of the filter minimization problem with only a polynomial number of constraints and which yields a concise nonlinear integer programming formulation. 
Since employing nonlinear constraints is generally computationally inefficient,
we `flatten' these nonlinear constraints leading to (1) an integer linear
programming and (2) a Boolean satisfaction formalism---both involve expressions
of constraints that are of comparable size.
\gobble{But, unlike the one-shot optimization scheme used by both linear and nonlinear
programming, Boolean satisfaction solves the problem by iteratively enumerating
the integer values of the objective. Since the domain of the objective for
filter minimization is small and the constraints in conjunctive normal form are
relatively efficient to be exploited, Boolean satisfaction outperforms all
other approaches in the experiment.}
Looking at the experimental results in depth, we observed that many constraints
remain inactive during the iterative search for a minimal filter.  
To speed computation within the Boolean
satisfaction formalism,
we treat these constraints just-in-time: only introducing a constraint when
we detect that some proposed variable assignment would violate it.  Empirical
results show that the speedup from this treatment can outweigh the
overhead of detecting and dynamically introducing the constraints. 

After presenting the problem statement in Section~\ref{sec:prob}, we first
formalize filter minimization as an integer nonlinear program in
Section~\ref{sec:INP}.  Based on that, we give integer linear
programming and Boolean
satisfaction formulations in Section~\ref{sec:ILP_SAT}. Experimental results
appear in Section~\ref{sec:expr}. 

\section{Problem Description}
\label{sec:prob}
\subsection{Combinatorial filters and their minimization}
We firstly recall the notion of a combinatorial filter:
%
%
\begin{definition}[procrustean filter~\citep{setlabelrss}]
A \defemp{procrustean filter}, \defemp{p-filter} or \defemp{filter} for short,
is a 6-tuple $(V, V_0, Y, \tau, C, c)$ in which $V$ is a non-empty finite set of states, $V_0$ is the set of initial states, $Y$ is the set of observations,
$\tau: V\times V\rightarrow 2^Y$ is the transition function, $C$ is the set of outputs (colors), and $c: V\to 2^{C}\setminus\{\emptyset\}$ is the output function.
\end{definition}

The sets of states, initial states, and observations for filter $F$ are denoted 
$V(F)$, $V_0(F)$ and $Y(F)$, respectively. Without loss of generality, we treat a
filter as a graph with states as its vertices and transitions as directed edges.

Given a filter $F=(V, V_0, Y, \tau, C, c)$, an observation sequence (or a
string) $s=y_1y_2\dots y_n\in Y^{*}$, and states $v, w \in V$, we say that $w$ is
\defemp{reached by} $s$ (or $s$ \defemp{reaches} $w$) when traced from $v$, if there exists a sequence of states $w_0,w_1,
\dots, w_{n}$ in $F$, such that $w_0 = v$, $w_n = w$, and $\forall i\in \lbrace 1, 2, \dots, n\rbrace,
y_i\in \tau(w_{i-1}, w_i)$.
We denote the set of all states reached by $s$ from a state $v$ in
$F$ with $\reachedvf{v}{F}{s}$ (also called $s$-children of $v$), and denote all states reached by $s$ from any initial state of the filter with $\reachedf{F}{s}$, 
i.e., $\reachedf{F}{s} = \bigcup_{v_0 \in V_0} \reachedvf{v_0}{F}{s}$.
If $\reachedvf{v}{F}{s}=\emptyset$, then we say that string $s$ \emph{crashes} in $F$ starting from $v$.
%
We also denote the set of all strings reaching $w$ from some initial state in $F$ as
$\reaching{F}{w}=\{s\in Y^{*}| w\in\reachedf{F}{s}\}$. 
The set of all strings that do not crash in $F$ is called
the \defemp{interaction language} (or, briefly, just \defemp{language}) of $F$,
and is written as $\Language{F}=\{s\in Y^{*}| \reachedf{F}{s}\neq\emptyset\}$.
We also use $\reachedc{F}{s}$ to denote the set of outputs for all states reached in $F$ by $s$, i.e., $\reachedc{F}{s}=\cup_{v\in \reachedv{F}{s}} c(v)$. For empty string $\epsilon$, we have $\reachedc{F}{\epsilon}=\cup_{v_0\in V_0} c(v_0)$. 

\gobble{
\begin{definition}[filter output] Given any filter $F=(V, V_0, Y, \tau, C, c)$, a string $s$ and an output $o\in C$, we say that $o$ is a \defemp{filter output} with input string $s$, if $o$ is an
output from the state reached by $s$, i.e., $o\in \cup_{v\in \reachedv{F}{s}} c(v)$.
We denote the set of all filter outputs for string $s$ as
$\reachedc{F}{s}=\cup_{v\in \reachedv{F}{s}} c(v)$.
\end{definition}

Specifically, for the empty string $\epsilon$, we have
$\reachedc{F}{\epsilon}=\cup_{v_0\in V_0(F)} c(v_0)$. 
}

\begin{definition}[output simulating]
Let $F$ and $F'$ be two filters, then $F'$ \defemp{output simulates} $F$ if
$\forall s\in \Language{F}$, $\reachedc{F'}{s}\neq \emptyset$ and
$\reachedc{F'}{s}\subseteq \reachedc{F}{s}$.
\end{definition}

Informally, one filter output simulates another if it admits all strings from the other filter and does not generate any new outputs for each of those strings.

We focus on filters with deterministic behavior:
\begin{definition}[deterministic]
A filter $F=(V, V_0, Y, \tau, C, c)$ is \defemp{deterministic} or
\defemp{state-determined}, if $|V_0|=1$, and for every $v_1, v_2, v_3\in V$ with
$v_2\neq v_3$, $\tau(v_1, v_2)\cap \tau(v_1, v_3)=\emptyset$. Otherwise, we say
that the filter is \defemp{non-deterministic}.
\end{definition}
Algorithm~$2$ in \citep{saberifar18pgraph} can be used to make any non-deterministic filter into a deterministic one. 
Then the filter minimization problem can be formalized as follows:
\ourproblem{\textbf{Filter Minimization (\fm)}}
{A deterministic filter $F$.}
{A deterministic filter $F^{\dagger}$ with fewest states, such that
$F^{\dagger}$ output simulates $F$.
}

We have also a filter reduction problem as follows.
\ourproblem{\textbf{$k$-Filter Reduction ($k$-\fm)}}
{A deterministic filter $F$.}
{A deterministic filter $F^{\dagger}$ with no more than $k$ states, such that
$F^{\dagger}$ output simulates $F$.
}

From now on, by filter we mean a deterministic filter, and we shall also assume $F$ has no unreachable states.

\section{Nonlinear integer programming formulation}
\label{sec:INP}
In our previous work~\citep{zhang20cover}, we expressed the two requirements on
the output, namely being both deterministic and output simulating of $F$, via
compatibility relationships which characterize all sets of states that can be
merged.  But since the set of compatibility relationships can be exponentially
large, they are time-inefficient to enumerate and space-inefficient to
represent explicitly.  Hence, in this section, we aim for concision:  instead
of enumerating compatibility relationships, we represent filters via vertex
covers. These covers can be encoded with a polynomial number of binary
variables, they allow for determinism and output simulation requirements to be
expressed, and they allow us to solve \fm as an instance of an integer nonlinear
program.

\subsection{The vertex cover representation of a filter}

%
To begin, define the basic combinatorial object involved:
\begin{definition}[vertex cover]
A \defemp{vertex cover} $\bm{K}=\{K_1, K_2, \dots, K_m\}$ on a filter $F$ is a
collection of subsets of vertices which cover all $F$'s vertices,
i.e., $K_i\subseteq V(F)$ for each $i$, and $\bigcup_{i=1}^{m} K_i =
V(F)$.
The size of $\bm{K}$ is number of the subsets, i.e., $|\bm{K}|=m$. 
\end{definition}


A vertex cover $\bm{K}=\{K_1, K_2, \dots, K_m\}$ on filter $F$ is \emph{\zipped} if for every subset $K_i \in \bm{K}$ and for each observation $y\in Y(F)$, there exists at least one subset $K_j \in \bm{K}$ that contains all $y$-children of the states within $K_i$.
Next, we show how a \zipped vertex cover begets a filter.

\begin{definition}[induced filter]
\label{def:inducedFilter}
Given a \zipped vertex cover $\bm{K}=\{K_1, K_2, \dots\}$ on $F=(V, \{ v_0 \},
Y, \tau, C, c)$, its induced filter $F^{\dagger}=(V^{\dagger}, \{\vdagger_0\}, Y, \tau^{\dagger}, C, c^{\dagger})$ is
constructed as follows:
\begin{enumerate}
\item Create a state $\vdagger_i$ in $F^{\dagger}$ for each non-empty subset $K_i$.
\item \label{def:indFilt_init} Select an arbitrary vertex $\vdagger_i$ as $\vdagger_0$, such that the corresponding
$K_i$ contains the initial state $v_0$ in $F$. 
\item \label{def:indFilt_trans} For each vertex $\vdagger_i$ and $y\in Y$, if $y$-children of vertices in $K_i$ is not empty, then add one transition from $\vdagger_i$ to $\vdagger_j$ under $y$ such that $K_j$ contains all $y$-children of $K_i$; if there are multiple such $\vdagger_j$s, pick an arbitrary one.  
\item Assign the output for $\vdagger_i$ to be 
$c^{\dagger}(\vdagger_i)\in \bigcap_{v\in K_i}c(v)$,
i.e., an output common to all vertices in $K_i$.
\end{enumerate}
\end{definition}

Note that each vertex in filter $F$
may be contained in multiple subsets in the vertex cover, and that each subset in the
cover is mapped to a \defemp{unique} vertex in the induced filter. Hence, we say
that each vertex in $F$ may be \defemp{mapped to} multiple vertices in the induced
filter. 
%
%

The construction does not shrink the filter's language.
\begin{lemma}
    \label{lem:langExpByCover}
    Let $F^{\dagger}$ be the induced filter for a \zipped vertex cover $\bm{K}$ on a filter $F$. It holds that $\Language{F} \subseteq \Language{F^{\dagger}}$.
\end{lemma}
\begin{sproof}
  This lemma can be proved by induction on the length of strings $s \in L(F)$ that shows if $s$ reaches a state $v$ in $F$, then $s$ reaches a state $\vdagger_i$ in $F^{\dagger}$ such that the corresponding $K_i$ contains $v$. This will show that if $s \in \Language{F}$, then $s \in \Language{F^{\dagger}}$, meaning that $\Language{F} \subseteq \Language{F^{\dagger}}$.
  %
\end{sproof}
%
The next throws light on why vertex covers interest us.

\begin{lemma}
\label{lm:vertexCoverForFM}
Given an input filter $F=(V, \{v_0\}, Y, \tau, C,
c)$, if there exists a solution to $k$-\fm, then there is always a filter $F^{\dagger}$ as a solution to $k$-\fm such that $F^{\dagger}$ is induced from a \zipped vertex cover on $F$.  
\end{lemma}
\begin{proof}
Let $F^{\star}=(V^{\star}, \{v^{\star}_0\}, Y, \tau^{\star}, C, c^{\star})$ be any
solution to $k$-\fm with input $F$. From $F^{\star}$ and $F$, we construct a \zipped vertex cover $\bm{K}$ on $F$, then
construct an induced filter $F^{\dagger}$ from $\bm{K}$ and show that
$F^{\dagger}$ is also a solution for $k$-\fm.

We construct $\bm{K}$ as follows:
For each $i \in \lbrace 1, 2, \ldots, |V^{\star}| \rbrace$, we choose $v^{\star}_i \in V^{\star}$ and let $K_i = \lbrace v \in V \mid \reaching{F}{v} \cap \reaching{F^{\star}}{v^{\star}_i} \neq \emptyset  \rbrace$. Then we set $\bm{K}  = \lbrace K_1, K_2, \ldots, K_{|V^{\star}|} \rbrace$. Collection $\bm{K}$ is a vertex cover on $F$ because by assumption, $\Language{F} \subseteq \Language{F^{\star}}$, which implies that for each $v \in V$, there is at least one vertex $v^{\star} \in V^{\star}$ such that $\reaching{F}{v} \cap \reaching{F^{\star}}{v^{\star}} \neq \emptyset$, and this means that each vertex $v$ of $F$ is contained in at least one subset $K \in \bm{K}$. In addition, $\bm{K}$ must be \zipped. For otherwise, some string is in $F$ but not in $F^{\star}$, which contradicts with the fact that $F^{\star}$ output simulates $F$.

Now, we show that $F^{\dagger}$ is also a solution to $k$-\fm.
Trivially, $|\bm{K}|=|V^{\star}|\leq k$, and by the construction in Definition~\ref{def:inducedFilter},  $|V(F^{\dagger})|\leq |\bm{K}|$. Hence, $|V(F^{\dagger})|\leq k$.

We will prove by contradiction that $F^{\dagger}$ output simulates $F$.
Suppose $F^{\dagger}$ does not output simulate $F$. 
Then there must be a string
$s\in \Language{F}$, such that either ($i$) $s\not\in F^{\dagger}$, or ($ii$)
at least two states are reached by $s$ in $F^{\dagger}$ ($F^{\dagger}$ is non-deterministic), or ($iii$)
$\reachedc{F^{\dagger}}{s}\not\subseteq \reachedc{F}{s}$. 
Regarding case ($i$),
since $s\in \Language{F}$ and $\Language{F}\subseteq \Language{F^{\star}}$,
we have $s\in \Language{F^{\star}}$. 
Let $v^{\star}_j$ be a vertex reached by $s$ in $F^{\star}$. 
Then $s$ must
also reach $v^{\dagger}_j$ in $F^{\dagger}$, which indicates that $s\in F^{\dagger}$.
Regarding case ($ii$), let $v^{\dagger}_j$ and $v^{\dagger}_l$ ($j\neq l$) be two states in
$F^{\dagger}$ that are reached by $s$.
Notice that there is an injective function from the vertices and
edges in $F^{\dagger}$ to those in $F^{\star}$. 
Thus, $s$ must reach two different
vertices $v^{\star}_j$ and $v^{\star}_l$ in $F^{\star}$, which contradicts the fact that $F^{\star}$ is
deterministic. 
Regarding case $(iii)$, there must exist a vertex $v^{\dagger}_i$ in $F^{\dagger}$ and a
vertex $v$ in $F$ that are both reached by $s$ and that $v^{\dagger}_i$ and $v$ have different outputs. 
But according
to the construction of $F^{\dagger}$, $v^{\dagger}_i$ must share the same output as
$v$. Hence, $F^{\dagger}$ must output simulate $F$. 
\end{proof}
Hence, to solve \fm, we can always look for vertex covers.
%

\subsection{Searching over vertex covers via variables} 
Now we represent a vertex cover with binary variables.

To encode a vertex cover $\bm{K}=\lbrace K_1, K_2, \dots K_m\rbrace$ on an input filter
$F=(V, \{v_0\}, Y, \tau, C, c)$, we introduce the following binary variables:
\begin{itemize}
\item Create a binary variable $R^{i}_v$ for each $v\in V$ and each $i\in \lbrace 1,2,\dots,|V|\rbrace$, and assign $R^{i}_v=1$ if and only if $v$ is contained in $K_i$. If $i>|\bm{K}|$, then we view $K_i$ as an empty set and set $R^i_v=0$ for all $v\in V$.
\item Create a binary variable $q^i$ for each $i\in \lbrace 1,2,\dots, |V|\rbrace$, and assign $q^i=1$ if and only if $K_i$ is not empty. 
\end{itemize}

We also define additional variables with constant values assigned from the structure of the input filter:
\begin{itemize} 
\item Introduce a binary variable $t^y_v$ for each $v\in V$ and \mbox{$y\in Y$}, to which we assign value $1$ if and only if $v$ has non-empty $y$-children.
\item Introduce a binary variable $p^o_v$ for each $v\in V$ and \mbox{$o\in O$}, to which we assign value $1$ if and only if $v$ has $o$ in its outputs, i.e., $o\in c(v)$.
\end{itemize}
With these variables, we can encode an output filter and the constraints for it to be a valid solution in \fm. 
\subsection{\fm as an integer nonlinear program (\INP)}
Now, we formalize \fm as an integer nonlinear program.
%
In what follows, we
denote the input filter as $F=(V, \{v_0\}, Y, \tau, C, c)$, the vertex cover to be searched for as
$\bm{K}$, and the induced output filter from $\bm{K}$ as $F^{\dagger}=(V^{\dagger},
\{\vdagger_0\}, Y, \tau^{\dagger}, C, c^{\dagger})$. 
%
%
For brevity, we will simply write $\forall i$ for $\forall i\in \{1,2,\dots, |V|\}$, $\forall v$ for $\forall v\in V$, and $\forall y$ for $\forall y\in Y$. 

\begin{mdframed}
\small
	\noindent Minimize\vspace{-0.5em}
	\begin{equation}\label{eq-0}
		\sum_{1\leq j\leq |V|} q^j
		\tag{\INP-Obj}
	\end{equation}
	Subject to:\vspace{-0.5em}
    	\begin{equation}\label{eq:ilp_variables}
			q^i, R_v^i \in \lbrace 0, 1 \rbrace: \forall i, \forall v 
			\tag{\INP-Vars}					
		\end{equation}      			   			
		\begin{equation} \label{eq-1}
			R^i_v\leq q^i:  \forall  i, \forall v
			\tag{\INP-NESubset}		
		\end{equation}		    
		\begin{equation}\label{eq-5}
			q^i \leq q^{i-1}: \forall i 
			\tag{\INP-Sym}					
		\end{equation}     
		\begin{equation}\label{eq-2}
			\sum_{1 \leq j \leq |V|} R^j_{v_0}\geq 1 
			\tag{\INP-ValidCover}
		\end{equation}  
		\begin{equation*}\label{eq-3}
		\footnotesize
		\begin{aligned}
		{\sum_{1\leq j\leq |V|} \prod_{v\in V} (2- R^i_v - t^y_v +R^j_{v_y})\geq 1: \forall  i, \forall y}
		\end{aligned}
		\small
		\tag{{\INP-\Zip}}
		\end{equation*} 
		\begin{equation}\label{eq-4}
		\small
			\sum_{o\in C}\prod_{v\in V} (1-R^i_v+p^o_v)\geq 1:  \forall i 
			\tag{\INP-Out}
		\end{equation} 
\end{mdframed}

The objective~\eqref{eq-0} is to minimize the number of non-empty subsets in $\bm{K}$.
For each $j$, variable $q^j$ receives value $1$ if at least one vertex of $F$ is assigned to $K_j$. This is
expressed by constraints~\eqref{eq-1}. 
We use the idea of M{\'e}ndez-D{\'\i}az and Paula~\citep{mendez2008cutting} to reduce symmetry
by pushing the non-empty subsets to smaller indices. This is imposed by constraints~\eqref{eq-5}.

Constraint~\eqref{eq-2} requires that the initial state of $F$ be contained in at least one subset of the vertex cover.
Together with constraints~\eqref{eq-3}, this ensures that all vertices of $F$ that are reachable from the initial state will be covered by $\bm{K}$.
For the output filter to be deterministic, for each observation $y$ and each state $\vdagger_i$ in $F^{\dagger}$, $\vdagger_i$ must have at most a single $y$-child. Accordingly, for each subset $K_i$, there must exist a subset $K_j$ that contains all the $y$-children of $K_i$.
More exactly, $\forall i\in \lbrace 1,2,\dots, |V|\rbrace$, $\forall y\in Y$:
\begin{equation*}
    \small
	\begin{aligned}
		\label{eq-3-explaination}
		&\exists j\in \lbrace 1,2,\dots,|V|\rbrace, \text{s.t.}, \underbrace{\forall v\in V, \left((R_v^i t_v^y =1)  \implies (R^j_{v_y}=1)\right)}_{\text{\footnotesize all $y$-children of $K_i$ are contained in $K_j$}}. \\
	\end{aligned}
\end{equation*}
In algebraically simplified form, this is the zipped constraints~\eqref{eq-3}.
Note that we allow multiple such $K_j$'s to exist for any $K_i$ and $y$, but in the output filter, only one will be (arbitrarily) picked for the transition.

In addition, the output for vertex $\vdagger_i$ should be the common output of all
vertices in the subset $K_i$. This means that for each subset $K_i$, all states
within that subset must share a common output. Formally, $\forall i\in \lbrace 1,2,\dots, |V|\rbrace$, 
\begin{equation*}
\small
	\begin{aligned}
		\label{eq-4-explaination}
		\exists o \in C, s.t., \underbrace{\forall v\in V, \left((R_v^i = 1)
\implies (p_v^o = 1)\right)}_{\text{\footnotesize all vertices in $K_i$ must share output $o$}},\\ 
	\end{aligned}
\end{equation*}
which is expressed by constraints~\eqref{eq-4}.

With a solution to the 
integer nonlinear program in hand,
we first form the vertex cover $\bm{K}$ by constructing the subsets according
to the values assigned to variables $R_v^i$'s. 
Then we make the output filter $F^{\dagger}$ by following Definition~\ref{def:inducedFilter}.

The next we prove the correctness of \INP.
%



\begin{lemma}[correctness]
    \label{lem:optimalSolINP}
    Let $\bm{K}$ be the vertex cover formed by an optimal solution to the
    integer nonlinear program for an input filter $F$ and let $F^{\dagger}$ be an
    induced filter from $\bm{K}$. Filter $F^{\dagger}$ is an optimal solution
    to \fm with input $F$.
\end{lemma}
\begin{sproof}
The constraints from the nonlinear programming formalize \emph{exactly} the induced filter $F^{\dagger}$ to be deterministic and output simulate the input filter $F$. And the objective function characterizes \emph{exactly} that $F^{\dagger}$ has the minimum number of states.
To show this, we establish the equivalence between the nonlinear constraints and properties of determinism and output simulating in \fm. This holds in both directions.

$\impliedby$: If constraints~\eqref{eq-2} and~\eqref{eq-3} hold, then $\bm{K}$ is a \zipped vertex cover, $F^{\dagger}$ constructed following Definition~\ref{def:inducedFilter} is deterministic, and $\Language{F}\subseteq\Language{F^{\dagger}}$ as per Lemma~\ref{lem:langExpByCover}. If constraint~\eqref{eq-4} is satisfied, then $\forall s\in \Language{F}$, $\reachedc{F}{s}\supseteq \reachedc{F^{\dagger}}{s}$. Hence, $F^{\dagger}$ is deterministic and output simulates $F$. 

$\implies:$ Given an $F^{\dagger}$ that is an optimal solution for \fm, construct an induced \zipped cover $\bm{K}$ following Lemma~\ref{lm:vertexCoverForFM}. The values of the variables encoding this cover must satisfy constraints~\eqref{eq-1} and~\eqref{eq-2}. If $F^{\dagger}$ is deterministic, then constraints~\eqref{eq-3} must also be satisfied. The fact that $F^{\dagger}$ output simulates $F$ implies that constraints~\eqref{eq-4} are satisfied.

Proof that if $F^{\dagger}$ is minimal, then the value of \eqref{eq-0} must be optimal (and vice versa) follows similarly. 
\end{sproof}

\section{Integer linear programming and SAT}
\label{sec:ILP_SAT}
In this section, we introduce three additional formulations of \fm: ($i$)
an integer linear programming formulation, ($ii$) a Boolean satisfaction formulation, and
($iii$) a Boolean satisfaction with \zipped constraints being added just-in-time.

\subsection{Integer linear programming (\IP) with linear constraints}
This section presents an integer linear program by linearizing the nonlinear constraints~\eqref{eq-3} and~\eqref{eq-4}.

To linearize constraints~\eqref{eq-3}, we introduce a binary variable
$a^{i,j}_{y}$ for each $i, j \in \lbrace 1,2, \dots,|V|\rbrace$ and $v \in V$ to determine whether there is a transition from vertex $\vdagger_i$ to vertex $\vdagger_j$ under label $y$ in the output filter $F^{\dagger}$. If $a^{i,j}_y=1$, then the value of term $\prod_{v\in V} (1-R^i_v+1-t^y_v+R^j_{v_y})$ must be a positive integer. Otherwise, we choose not to build such a transition in the output filter, regardless of the value of the corresponding term. Mathematically, we have: 
\begin{equation}
    \small
	\begin{aligned}
		\label{eq-3-1-ip}
		{a^{i,j}_y+R^i_v+t^y_v-R^j_{v_y}\leq 2:  \forall i, \forall j, \forall v, \forall y}.
	\end{aligned}
	\tag{\IP-\Zip-1}
\end{equation}

Then constraints~\eqref{eq-3} are written as 
\begin{equation}
    \small
	\begin{aligned}
		\label{eq-3-2-ip}
		\sum_{j=1}^{|V|} a^{i,j}_y\geq 1: \forall i, \forall y.
	\end{aligned}
	\tag{\IP-\Zip-2}
\end{equation}

For constraints~\eqref{eq-4}, we similarly introduce a binary variable $b^i_o$, with value $1$ to denote the fact that the term $\prod_{v\in V} (1-R_v^i+p^o_v)$ has a positive value. If $b^i_o=0$, then we do not care whether the value of the corresponding term is positive or not. 
Thus, we add the following constraints:
\begin{equation}
    \small
	\begin{aligned}
		\label{eq-4-1-ip}
		1-b^i_o+1-R^i_v+p^o_v\geq 1: \forall i, \forall o, \forall v.
	\end{aligned}
	\tag{\IP-Out-1}
\end{equation}

Then constraints~\eqref{eq-4} are linearized as follows:
\begin{equation}
    \small
	\begin{aligned}
		\label{eq-4-2-ip}
		\sum_{o\in C} b^i_o\geq 1: \forall i.
	\end{aligned}
	\tag{\IP-Out-2}
\end{equation}

%
%
%

\subsection{Boolean satisfaction (\SAT)}
We next treat \fm as a sequence of $k$-\fm problems by enumerating the bound on the output filter size.
Each $k$-\fm is formalized as a Boolean satisfaction problem, which we call $\SAT_{[k]}$. To find the minimal
filter, the idea is to solve a $\SAT_{[k]}$, and then decrement $k$ until no smaller output filter can be found.

To obtain a $\SAT_{[k]}$ instance, we first remove variables $q^i$ ($\forall i$) and constraints~\eqref{eq-1} and~\eqref{eq-5} since we do not need~\eqref{eq-0}
and only want to find an output filter with a size bounded by $k$. Next, we treat binary variables $R^i_v$, $a^{i,j}_y$,
$b^i_{o}$ as boolean-valued and write constraints~\eqref{eq-3-1-ip}--\eqref{eq-4-2-ip}
in conjunctive normal form (CNF). 

Given a filter minimization problem with size bounded by~$k$,
constraint~\eqref{eq-2}
is written as a clause: 
\begin{equation}
\small
\label{eq-2-sat}
\blor_{i\in \lbrace 1,2,\dots, k\rbrace} R^i_{v_0}.
\tag{\SAT-ValidCover}
\end{equation}

Constraints~\eqref{eq-3-1-ip} and~\eqref{eq-3-2-ip} are written as the following
clauses: 
\vspace*{-4pt}
\begin{equation}
\small
\label{eq-3-1-sat}
\begin{aligned}
\negation{a^{i,j}_y}\lor \negation{R^i_v}\lor \negation{t^y_v}\lor R^j_{v_y}: \forall i, \forall j, \forall v, \forall y
\end{aligned}
\tag{\SAT-\Zip-1}
\end{equation}
\vspace*{-10pt}
\begin{equation}
\small
\label{eq-3-2-sat}
\begin{aligned}
\blor_{j \in \lbrace 1,2,\dots, k\rbrace}a^{i,j}_y: \forall i, \forall y.
\end{aligned}
\tag{\SAT-\Zip-2}
\end{equation}
\vspace*{-4pt}

And constraints~\eqref{eq-4-1-ip} and~\eqref{eq-4-2-ip} become:
\begin{equation}
\small
\label{eq-4-1-sat}
\begin{aligned}
\negation{b^{i}_o}\lor \negation{R^i_v}\lor p^o_v: \forall i, \forall o, \forall v,
\end{aligned}
\tag{\SAT-Out-1}
\end{equation}
\begin{equation}
\small
\label{eq-4-2-sat}
\begin{aligned}
\blor_{o\in C}b^{i}_o: \forall i.
\end{aligned}
\tag{\SAT-Out-2}
\end{equation}

Notice that consecutive \SAT instances share most of their variables
and constraints. 
The $\SAT_{[k]}$ instance is equivalent to the $\SAT_{[k+1]}$ one but with
additional unit clauses $\negation{R^{k+1}_{v}}$ for all $v\in V$. 
%
%
Instead of making each
$\SAT_{[k]}$ instance from scratch and solving it, we add unit clauses while 
decreasing $k$.
This allows the solver to re-use knowledge acquired from previous \SAT
instances, and leads to an incremental anytime procedure in
Algorithm~\ref{alg:incrementSAT}.
First, we initialize $k$ to be $|V(F)|$ (line~$1$). 
Next, we construct a CNF formula for $\SAT_{[k]}$ and invoke the SAT solver (lines~$2$--$6$). 
If an assignment is found for $\SAT_{[k]}$ within the time budget (line~$7$), then we set its cover choice to be the smallest one found so far (line~$8$), update the time
budget by the amount of time used in this iteration (line~$9$), add the unit clauses (line~$10$), and decrease $k$ (line~$11$). Otherwise, if the $\SAT_{[k]}$ has not been solved within the time budget, then we use
the minimum cover found so far to construct the minimal filter (line~$16$). 
When given an adequate time budget, the algorithm will find the minimal filter.
%
And running the algorithm for a longer duration increases the chance of finding a smaller filter.

\setlength{\textfloatsep}{0pt}
\begin{algorithm}[ht]
{\small
\caption{$\mathrm{\SAT}(F, timeout)$}
\label{alg:incrementSAT}
\begin{algorithmic}[1] 
\STATE $k\gets |V(F)|$
\STATE $\operatorname{CNF}\gets \operatorname{BuildFormula}(F, k)$ 
\STATE Initialize minimum vertex cover $\bm{K}_{min}$  to be empty
\STATE $solver\gets\operatorname{SATSolver(CNF)}$
\WHILE{$k\geq 1$ and $timeout>0$}
	\STATE $result\gets solver.\operatorname{solve}(timeout)$
	\IF{$result.solved$}
		\STATE $\bm{K}_{min}\gets result.model$	
		\STATE Reduce $timeout$ by the time used
		\STATE Add unit clauses $\negation{R^k_v}$ ($\forall v\in V$) to $solver$
		\STATE $k\gets k-1$	
	\ELSE
		\STATE \textbf{break}
	\ENDIF
\ENDWHILE
\STATE $F'\gets \operatorname{FilterConstruction}(F, \bm{K}_{min})$
\RETURN $F'$
\end{algorithmic}
}
\end{algorithm}
%
%
\subsection{\SAT with just-in-time constraints: \LazySAT}
\label{sec:lazy}
In \SAT, \zipped constraints~\eqref{eq-3-1-sat} and~\eqref{eq-3-2-sat} are critical to ensure
deterministic transitions between subsets of $\bm{K}$. 
If a set of vertices in the input filter do not share any common output, then there is no need to check  these \zipped constraints on any set containing these vertices. In this case, we say that the zipped constraints related to these vertices are inactive. The existence of inactive constraints slows down the resolution of the \SAT problem, but detecting and representing all active \zipped constraints in \fm requires exponential time and space~\citep{zhang20cover}. To speed up the \SAT approach without significant overhead, we introduce \LazySAT, a just-in-time treatment of the \zipped constraints. In \LazySAT, we first partition these constraints into non-overlapping sets of clauses, solve the \SAT problem without these constraints, and introduce each set of clauses only when a non-\zipped cover is returned and it violates these clauses.

\Zipped constraints ensure that $\bm{K}$ covers $F$ as well as being \zipped.
%
To treat them lazily, we update constraints~\eqref{eq-2-sat} so that every state in $F$ is contained in at least a subset:
\begin{equation}
\small
\begin{aligned}
\bland_{v\in V} \bigg(\blor_{i\in \lbrace 1,2,\dots, k\rbrace} R^i_v\bigg).
\end{aligned}
\tag{\LazySAT-ValidCover}
\end{equation}

Next, we partition the clauses in the \zipped constraints. Let the set of clauses from constraint~\eqref{eq-3-1-sat} be $\mathbb{A}$.
Then $\mathbb{A}$ can be  partitioned into
non-overlapping subsets according to the vertex $v$ and outgoing label $y$, i.e.,
$\mathbb{A}=\cup_{v\in V}\cup_{y} \mathbb{A}^v_y$. Each subset $\mathbb{A}^v_y$ consists of the
following clauses: 
\begin{equation*}
\small
\bland_{i\in \lbrace 1,2,\dots, k\rbrace}\bland_{j\in \lbrace 1,2,\dots, k\rbrace} \big(\negation{a^{i,j}_y}\lor
\negation{R^i_v}\lor \negation{t^y_v}\lor R^j_{v_y}\big).
\end{equation*}
Similarly, the set of clauses from constraint~\eqref{eq-3-2-sat} is denoted as
$\mathbb{B}$, which is parameterized by the outgoing label $y$. Each subset
$\mathbb{B}_y$ consists of the following clauses: 
\begin{equation*}
\small
\bland_{i\in \lbrace 1,2,\dots, k\rbrace}\bigg(\blor_{j\in \lbrace 1,2,\dots, k\rbrace} a^{i,j}_y\bigg).
\end{equation*}

We detect the violation of these clauses, and add the clauses to the solver as needed. Let $Y_c$ be the set
of outgoing labels $y$ such that the clauses in $\mathbb{B}_y$ are already
present in the solver, and $P$ be the set of vertex $v$ and outgoing
label $y$ pairs such that $\mathbb{A}^v_y$ are also added to the solver. Both $Y$ and $P$ are initialized as empty sets. Given a vertex cover $\bm{K}$ returned from the \SAT algorithm, if $\bm{K}$ is not \zipped, then there must exist a set of states $K\in \bm{K}$ and a label $y$ such that all $y$-children of vertices in $K$ are not contained in any single subset in the cover $\bm{K}$. This must be a consequence of violating some missing clauses parameterized by $K$ and $y$ in the \zipped constraints. 
We add these clauses as follows: ($i$) if $y\not\in Y_c$, add $y$ to $Y_c$ and add clauses in
$\mathbb{B}_y$ to the solver; ($ii$) for any $v\in K$, if $(v,y)\not\in P$, add $(v,y)$ to $P$ and add clauses in $\mathbb{A}^v_y$ to the solver. Now, repeatedly call the solver, adding clauses if needed, until
we find a \zipped vertex cover with size no greater than $k$. To find a minimum solution, 
follow the same procedure as Algorithm~\ref{alg:incrementSAT}. 

\section{Experimental results}
\label{sec:expr}
We implemented \INP, \IP, \SAT and \LazySAT in Python, based on mixed integer
nonlinear solver SCIP~\citep{GamrathEtal2020ZR}, mixed integer linear solver Gurobi~\citep{gurobi}, and SAT solver
CaDiCaL~\citep{cadical2020}. All executions are conducted on an OSX laptop with a \SI{2.4}{\GHz} Intel Core
i5 processor, and each algorithm is given \SI{10}{\minute} budget to solve a filter minimization problem.

First, we minimize the filter in Figure~\ref{fig:motivating_input}.
\INP failed to give a result before timing out, while \IP, \SAT and \LazySAT give minimal filters with $5$ states within \SI{1}{\second}, \SI{2}{\second} and \SI{100}{\second}, respectively. One such minimal filter is shown in Figure~\ref{fig:motivating_minimal}. 
We collected no further results for \INP since it appears to be incapable of minimizing filters with more than $10$ states within \SI{10}{\minute}.

To test the performance of the remaining three algorithms, we randomly generated a filter as follows: ($i$)~construct a tree with a root node at layer $0$ and $w$ nodes at each of $d$ additional layers, and then connect each vertex from a parent vertex in an earlier layer by drawing a directed edge; ($ii$)~randomly pick $m$ vertices to add self loops; ($iii$)~randomly pick $n$ vertices to connect to some parent vertex in a later layer, so as to generate cycles. Next we randomly assign $n_o$ outputs to vertices in the filter, where each vertex is assigned with $p$ of them. Similarly, we randomly assign $n_y$ observations to the edges in the filter while keeping the filter deterministic. 

To compare \IP and \SAT-based approaches, we start with a 
filter structure randomly generated by parameters $d$=$4$, $w$=$3$, $m$=$n$=$2$, $p$=$2$ and $n_o$=$5$. For any given number of observations $n_y$, we sample $10$ filters, and collect the time to minimize these filters for each algorithm in Figure~\ref{fig:exp_obs}. As more observations are added to the filter, fewer states share common observations. The \zipped constraints~\eqref{eq-3-1-ip} and~\eqref{eq-3-1-sat} will be simplified, since $a^{i,j}_y$ connects with fewer vertices. Hence, the computational time for both \IP and \SAT-based approaches tend to decrease. Fixing the number of observations to be $n_y=5$, we also collect the computation time under varying outputs in Figure~\ref{fig:exp_output}. 
This gives an opposite trend as increasing the number of outputs makes the problem harder from two aspects: ($i$) the number of variables increases; ($ii$) the number of output constraints~\eqref{eq-4-1-ip} or~\eqref{eq-4-1-sat} increases owing to both an increasing number of outputs and an increasing number of vertices with $p^o_v=0$ for each output $o$.
%
Across both studies, \SAT-based approaches outperform integer linear programming. 
We speculate that this is because the constraints for \fm are fundamentally combinatorial in nature and can be concisely encoded in CNF. 
These CNF constraints can be exploited relatively efficiently (e.g., by building a constraint-dependency graph). And a final factor might
be that the objective function really takes a limited range of values and its values can be enumerated efficiently.

\begin{figure}
\begin{subfigure}[b]{0.49\linewidth}
\includegraphics[scale=0.41]{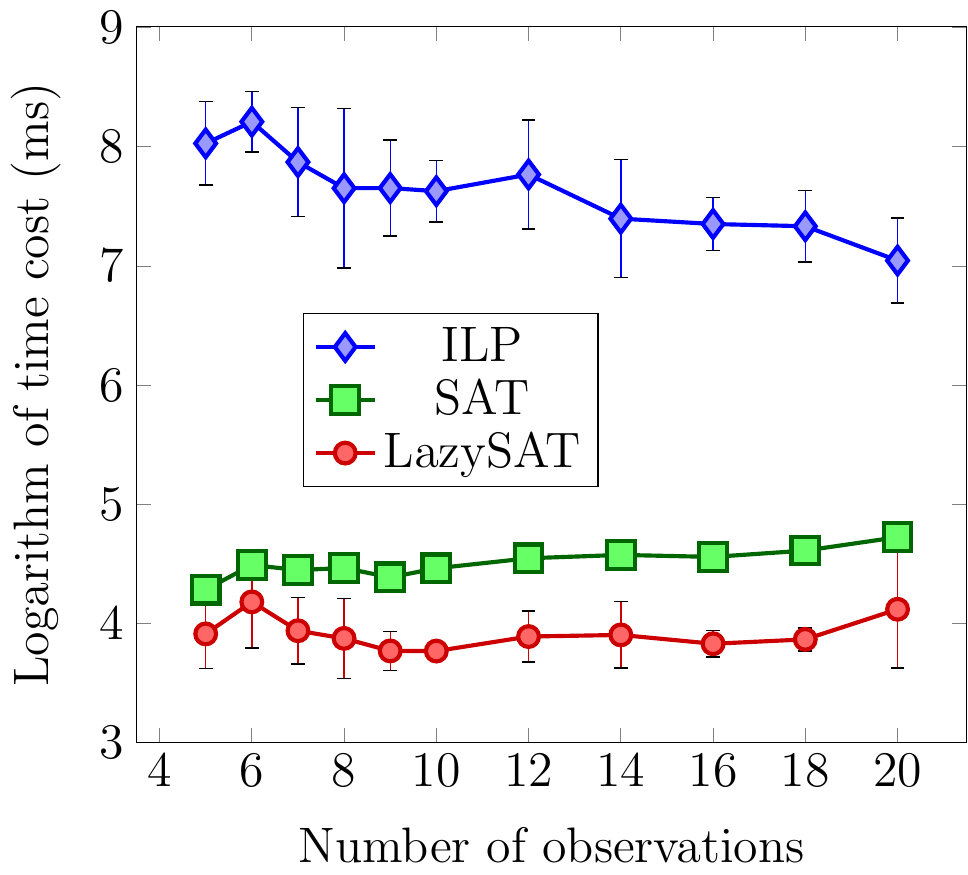}
\caption{Time cost (natural log) for inputs with varying observations.\label{fig:exp_obs}}
\end{subfigure}
\begin{subfigure}[b]{0.49\linewidth}
\includegraphics[scale=0.41]{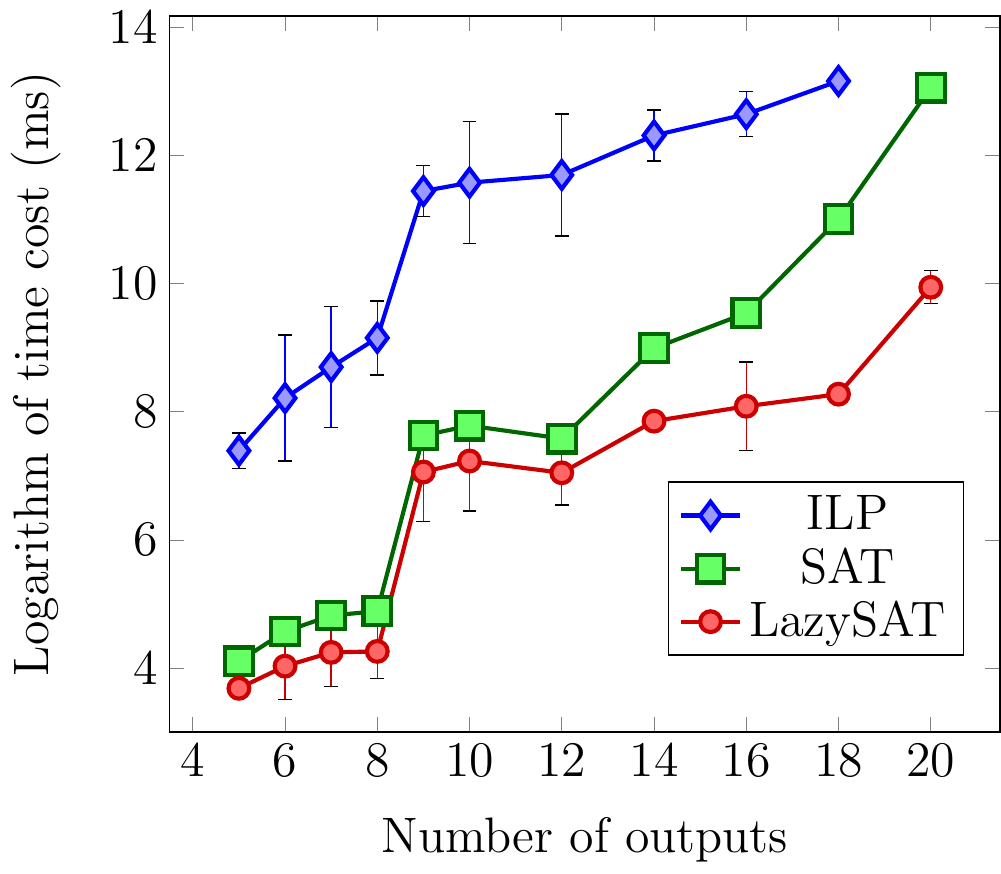}
\caption{Time cost (natural log) for inputs with varying outputs.\label{fig:exp_output}}
\end{subfigure}
\caption{Comparison of logarithmic computational time to minimize filters with different number of outputs and observations.}
\vspace{-0.4cm}
\end{figure}
\begin{figure}
\centering
\includegraphics[scale=0.56]{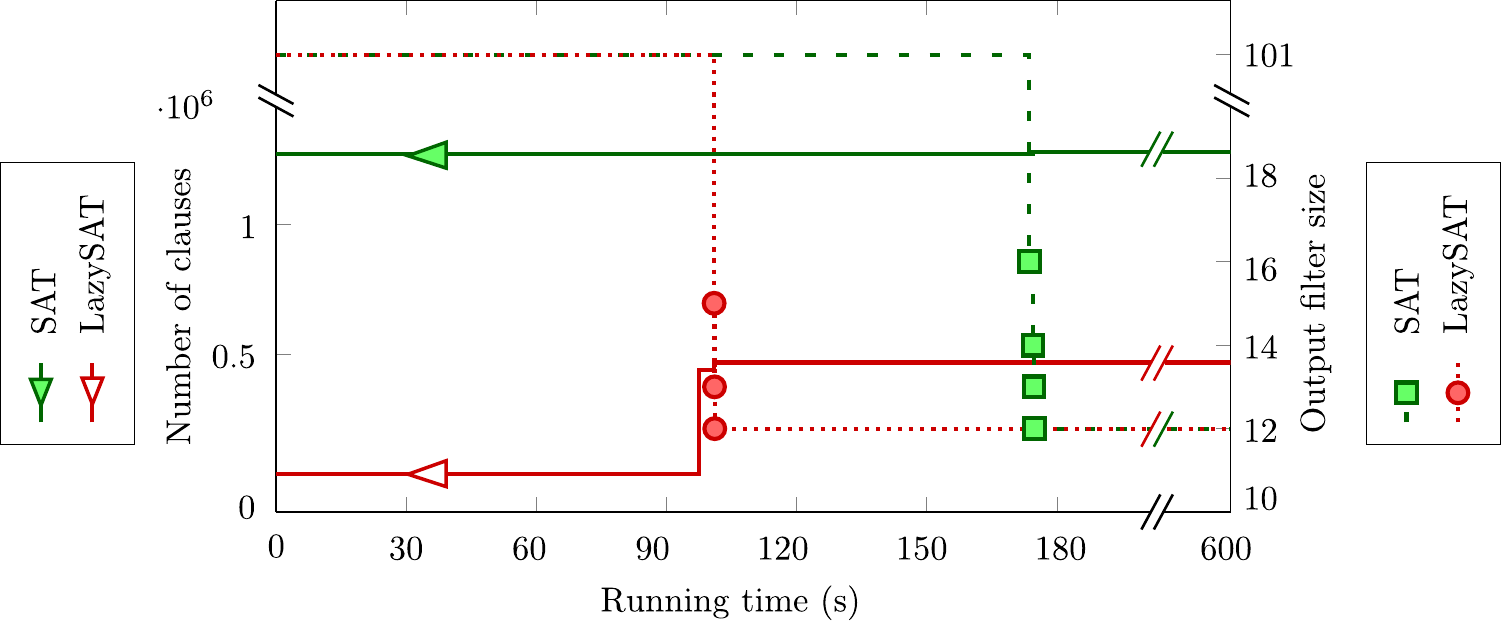}
\caption{The number of constraints used by \SAT and \LazySAT to find the sub-optimal solutions while increasing the running time. The input filter is constructed with parameters $d=20$, $w=5$, $m=n=10$, $p=1$, with $50$ observations and $5$ outputs.
\label{fig:exp_large}}
\end{figure}

Observe that in Figure~\ref{fig:exp_output}, as we increase the number of outputs, \LazySAT significantly outperforms \SAT since few states share common outputs, so most \zipped constraints are inactive and can be removed. 
We further tested them on a larger filter instance, where many
states share common outputs and hence a significant proportion of constraints
become active.
In Figure~\ref{fig:exp_large}, instead of presenting the time to find a minimal solution, we report the number of clauses used by the solver, and size of the sub-optimal solutions found by the two algorithms along the way. 
\LazySAT is still able to find sub-optimal solutions faster than \SAT, and the number of clauses used by \LazySAT is much fewer than those in \SAT. Treating constraints lazily does incur overhead in detecting and adding the active clauses, but the speedup from just-in-time treatment is seen to outweigh its overhead even when a large number of vertices share common outputs.

\section{Conclusion}
\label{sec:conc}
This paper accelerates filter minimization through constraints. It introduces a concise constraint description, encodes it in different forms, and reports empirical evidence suggesting that constraints in conjunctive normal form are most efficient. It also proposes a just-in-time treatment of constraints to speed up iterative filter reduction. Future work might consider  non-deterministic input, as well as searching for non-deterministic minimizers.

\bibliographystyle{IEEEtran}
\bibliography{mybib}
\end{document}